\documentclass{article}
\usepackage[utf8]{inputenc} 
\usepackage[T1]{fontenc}    
\usepackage{url}            
\usepackage{microtype}      
\usepackage{amsfonts}
\usepackage{amsmath,amsthm}

\usepackage{fullpage}
\usepackage{amssymb}
\usepackage{mathtools}

\newcommand{\cS}{\mathcal{S}}
\newcommand{\cT}{\mathcal{T}}
\newcommand{\cD}{\mathcal{D}}
\newcommand{\Id}{\mathrm{I}}

\newcommand{\R}{\mathbb{R}}
\newcommand{\E}{\mathbb{E}}
\newcommand \T{^\top}

\newcommand{\Tr}{\mathrm{Tr}}

\DeclareMathOperator*{\argmin}{\arg\!\min}

\newtheorem{theorem}{Theorem}
\newtheorem{lemma}{Lemma}

\theoremstyle{definition}

\theoremstyle{assumption}

\usepackage{microtype}
\usepackage{lmodern,amsmath}

\title{A Markov Chain Theory Approach to Characterizing the Minimax Optimality 
of Stochastic Gradient Descent 
(for Least Squares)}

\usepackage{authblk}
\author[1]{Prateek Jain}
\author[2]{Sham M. Kakade}
\author[2]{Rahul Kidambi}
\author[1]{Praneeth Netrapalli}
\author[2]{\\Venkata Krishna Pillutla}
\author[3]{Aaron Sidford}
\affil[1]{Microsoft Research, Bangalore, India,    \protect\url{{prajain,praneeth}@microsoft.com}}
\affil[2]{University of Washington, Seattle, WA, USA,  \\
\protect\url{{sham,pillutla}@cs.washington.edu},\ \protect\url{rkidambi@uw.edu}}
\affil[3]{Stanford University, Palo Alto, CA, USA, \protect\url{sidford@stanford.edu}.}
\newcommand{\norm}[1]{\left\lVert#1\right\rVert}
\begin{document}

\maketitle

\begin{abstract}
This work provides a simplified proof of the statistical minimax
optimality of (iterate averaged) stochastic gradient descent (SGD), for
the special case of least squares. This result is obtained by
analyzing SGD as a stochastic process and by sharply characterizing 
the stationary covariance matrix of this process.  The finite rate optimality characterization captures the
constant factors and addresses model mis-specification.
\end{abstract}

\section{Introduction}

Stochastic gradient descent is among the most commonly used
practical algorithms for large scale stochastic optimization. The
seminal result of ~\cite{Ruppert88,PolyakJ92} formalized this effectiveness,
showing that for certain (locally quadric) problems,
asymptotically, stochastic gradient descent is statistically minimax
optimal (provided the iterates are averaged).  There are a number of
more modern
proofs~\cite{Bach14,DieuleveutB15,DefossezB15,JainKKNS16} of
this fact, which provide finite rates of convergence. Other recent algorithms also achieve the statistically optimal
minimax rate, with finite convergence rates~\cite{FrostigGKS15}.

This work provides a short proof of this minimax optimality for SGD
for the special case of least squares through a characterization of
SGD as a stochastic process. The proof builds on ideas developed
in~\cite{DefossezB15,JainKKNS16}.

{\bf SGD for least squares.} 
The expected square loss for $w \in \R^d$ over input-output
pairs $(x,y)$, where $x\in \R^d$ and $y\in \R$ are sampled from
a distribution $\cD$, is: 
\[
L(w) = \frac{1}{2} \, \E_{(x,y)\sim\cD} [\left(y- w \cdot x \right)^2] 
\]
The optimal weight is denoted by: 
\[
w^* := \argmin_w L(w) \, .
\]
Assume the argmin in unique.

Stochastic gradient descent proceeds as follows: at each iteration
$t$, using an i.i.d.  sample $(x_t,y_t)\sim \cD$, the update of $w_t$ is:
\[
w_t = w_{t-1} + \gamma (y_t - w_{t-1} \cdot x_t) x_t
\]
where $\gamma$ is a fixed stepsize.

{\bf Notation.}
For a symmetric positive definite matrix $A$ and a vector $x$, define:
\[
\|x\|_A^2 := x^\top A x.
\]
For a symmetric matrix $M$, define the induced matrix norm under $A$ as:
\begin{align*}
\|M\|_A := \max_{\|v\|=1} \frac{v^\top M v}{v^\top A v} = \|A^{-1/2} M
A^{-1/2}\|.
\end{align*}

{\bf The statistically optimal rate.}
Using $n$ samples (and for large enough $n$), the minimax optimal rate is achieved by the maximum
likelihood estimator (the MLE), or, equivalently, the empirical risk
minimizer. Given $n$ i.i.d. samples $\{(x_i,y_i)\}_{i=1}^n$, define 
\[ 
\widehat w_n^{\textrm{MLE}} := \arg\min_w \frac{1}{n}\sum_{i=1}^{n}
\frac{1}{2}\left(y_i-w \cdot x_i \right)^2 
\] 
where $\widehat w_n^{\textrm{MLE}}$ denotes the MLE estimator
over the $n$ samples.

This rate can be characterized as follows: define
\[
\sigma^2_{\mathrm{MLE}} := 
\frac{1}{2}\E\left[ (y - w^*x)^2 \|x\|^2_{H^{-1}}\right]\, ,
\]
and the (asymptotic) rate of the MLE is 
$\sigma^2_{\mathrm{MLE}}/n$~\cite{lehmann1998theory,Vaart00}. Precisely, 
\[
\lim_{n\to\infty}\frac{\E[L(\widehat w_n^{\textrm{MLE}})]-L(w^*)}{\sigma^2_{\mathrm{MLE}}/n} = 1,
\]
The works of~\cite{Ruppert88,PolyakJ92} proved that a certain averaged
stochastic gradient method achieves this minimax rate, in the limit.

For the case of additive noise models (i.e. the ``well-specified''
case), the assumption is that $y=w^* \cdot x+\eta$, with $\eta$ being independent of
$x$. Here,  it is straightforward to see that:
\[
\frac{\sigma^2_{\mathrm{MLE}}}{n} = \frac{1}{2}\, \frac{d\sigma^2}{n}.
\]
The rate of $\sigma^2_{\mathrm{MLE}}/n$ is still minimax optimal
even among mis-specified models, where the additive noise assumption
may not hold~\cite{KushnerClark,lehmann1998theory,Vaart00}.

{\bf Assumptions.}
Assume the fourth moment of $x$ 
is finite. Denote the second moment matrix of $x$ as
\begin{align*}
H := \E[x x^\top] \, ,
\end{align*}
and suppose $H$ is strictly positive definite with minimal eigenvalue:
\[
\mu := \sigma_{\min}(H) \, .
\]
Define $R^2$ as the smallest value which
satisfies:
\[
\E[\|x\|^2 x x^\top] \preceq R^2 \E[x x^\top] \, .
\]
This implies
$\Tr(H)=\E\|x\|^2 \leq R^2$.

\section{Statistical Risk Bounds}

Define:
\[
\Sigma:=\E[ (y - w^* x)^2 x x^\top] \, ,
\]
and so the optimal constant in the rate can be written as:
\[
\sigma^2_{\mathrm{MLE}} = \frac{1}{2}\Tr(H^{-1}\Sigma) = 
\frac{1}{2}\E\left[ (y - w^*x)^2 \|x\|^2_{H^{-1}}\right]\, ,
\]
For the mis-specified case, it is helpful to define:
\[
\rho_{\mathrm{misspec}} := \frac{d\|\Sigma\|_H}{\Tr(H^{-1}\Sigma)} \, ,
\]
which can be viewed as a measure of how mis-specified the model is. Note
if the model is well-specified, then $\rho_{\mathrm{misspec}}=1$.

Denote the average iterate, averaged from iteration $t$ to $T$,  by:
\[
\overline{w}_{t:T} := \frac{1}{T-t} \sum_{t'=t}^{T-1} w_{t'} \, .
\]

\begin{theorem}\label{theorem:main}
Suppose $\gamma <\frac{1}{R^2}$. The risk is bounded as:
\begin{align*}
&\E[L(\overline{w}_{t:T})] - L(w^*) 
\leq \bigg( 
	\sqrt{ \frac{1}{2} \, \exp\big(- \gamma \mu t \big)  R^2 \|w_0 - w^*\|^2 }  +  
	\sqrt{\big( 1+\frac{\gamma R^2}{1-\gamma R^2} \rho_{\mathrm{misspec}} \big)
		\frac{\sigma_{\mathrm{MLE}}^2}{T-t}}  \, 
\bigg)^2 \, .
\end{align*}
\end{theorem}

The bias term (the first term) decays at a geometric rate (one can set
$t=T/2$ or maintain multiple running averages if $T$ is not known in
advance). If $\gamma = 1/(2R^2)$ and the model is well-specified
($\rho_{\mathrm{misspec}}=1$), then the variance term is
$2\sigma_{\mathrm{MLE}}/\sqrt{T-t}$, and the rate of the bias
contraction is $\mu/R^2$. If the model is not well specified, then
using a smaller stepsize of
$\gamma = 1/(2\rho_{\mathrm{misspec}} R^2)$, leads to the same minimax
optimal rate (up to a constant factor of 2), albeit at a slower bias
contraction rate.  In the mis-specified case, an example
in~\cite{JainKKNS16} shows that such a smaller stepsize is required in
order to be within a constant factor of the minimax rate.  An even
smaller stepsize leads to a constant even closer to that of the
optimal rate.

\section{Analysis}

The analysis first characterizes a bias/variance decomposition, where
the variance is bounded in terms of properties of the stationary
covariance of $w_t$. Then this asymptotic covariance matrix is analyzed. 

Throughout assume:
\begin{align*}
	\gamma < \frac{1}{R^2} \, .
\end{align*}

\subsection{The Bias-Variance Decomposition}
The gradient at $w^*$ in iteration $t$ is:
\[
\xi_t :=- (y_t - w^* \cdot x_t) x_t \, ,
\]
which is a mean $0$ quantity. Also define:
\[
B_t := \Id - x_t x_t^\top \, .
\]
The update rule can be written as:
\begin{align*}
w_t-w^* & = w_{t-1} -w^* + \gamma (y_t - w_{t-1} \cdot x_t) x_t\\
& = (\Id-\gamma x_t x_t^\top) (w_{t-1}-w^*) - \gamma \xi_t \\
&= B_t (w_{t-1}-w^*) - \gamma \xi_t  
\, .
\end{align*}
Roughly speaking, the above shows how the process on $w_t-w^*$
consists of a contraction along with an addition of a zero mean quantity.

From recursion,
\begin{align}
\label{eq:finalIterateExpansion}
w_t-w^* &= B_t \cdots B_1 (w_0-w^*) - \gamma \left(\xi_t  + 
  B_t\xi_{t-1} +\cdots+ B_t \cdots B_2 \xi_1\right) \,.
\end{align}

It is helpful to consider a certain bias and variance
decomposition. Let us write:
\[
\E[\|\overline{w}_{t:T} - w^*\|_H^2|\xi_0=\cdots =\xi_T=0]
: = 
\frac{1}{(T-t)^2} \E\left[\norm{\sum_{\tau=t}^{T-1}B_\tau \cdots B_1 (w_0-w^*)}^2_H \right]
\, .
\]
and
\[
\E[\|\overline{w}_{t:T} - w^*\|_H^2|w_0=w^*] 
= \left(\frac{\gamma}{T-t}\right)^2\cdot
\E\left[\norm{\sum_{\tau=t}^{T-1} \left(\xi_\tau  + B_\tau\xi_{\tau-1} +\cdots+ B_\tau \cdots B_2 \xi_1\right)}^2_H \right]
\]
(The first conditional expectation notation slightly abuses notation,
and should be taken as a definition\footnote{The abuse is due that the
right hand side drops the conditioning.}).

\begin{lemma}
The error is bounded as:
\begin{align*}
\E[L(\overline{w}_{t:T})] - L(w^*) 
\leq \frac{1}{2}
\bigg(
\sqrt{\E[\|\overline{w}_{t:T} - w^*\|_H^2|\xi_0=\cdots =\xi_T=0]} + 
\sqrt{\E[\|\overline{w}_{t:T} - w^*\|_H^2|w_0=w^*]}
\bigg)^2 \, .
\end{align*}
\end{lemma}

\begin{proof}
Equation~\ref{eq:finalIterateExpansion} implies that:
\begin{align*}
\overline{w}_{t:T}-w^* = \frac{1}{T-t}\sum_{\tau=t}^{T-1}B_\tau \cdots B_1 (w_0-w^*) - \frac{\gamma}{T-t}\sum_{\tau=t}^{T-1} \left(\xi_\tau  + B_\tau\xi_{\tau-1} +\cdots+ B_\tau \cdots B_2 \xi_1\right) \,.
\end{align*}
Now observe that for vector valued random variables $u \text{ and } v$, $(\E
u\T H v)^2 \le \E[\|u\|_H^2] \E[\|v\|_H^2]$ implies
\[
	\E \|u+v\|_H^2  \le \big(\sqrt{\E\|u\|_H^2 } + \sqrt{\E\|v\|_H^2}  \big)^2 \,,
\]
the proof of the lemma follows by noting that $\E[L(\overline{w}_{t:T}) - L(w^*)] = \frac 1 2 \E \|\overline{w}_{t:T} - w^*\|_H^2$.
\end{proof}

\noindent
{\bf Bias.\/}
The bias term is characterized as follows:

\begin{lemma}
For all $t$,
\[
\E[\|w_t - w^*\|^2|\xi_0=\cdots =\xi_T=0] \leq \exp(-\gamma \mu t) \|w_0 - w^*\|^2  \, .
\]
\end{lemma}

\begin{proof}
Assume $\xi_t=0$ for all $t$. Observe:
\begin{eqnarray*}
\E\|w_t - w^*\|^2
 & = &\E\|w_{t-1} - w^*\|^2 - 
2 \gamma (w_{t-1} -w^*)^\top \E[ xx^\top]  (w_{t-1} -w^*)\\
&&+\gamma^2 (w_{t-1} -w^*)^\top \E[\|x\|^2 xx^\top] (w_{t-1} -w^*)\\
& \leq &\E\|w_{t-1} - w^*\|^2 - 
2 \gamma (w_{t-1} -w^*)^\top H  (w_{t-1} -w^*)\\
&&+\gamma^2 R^2 (w_{t-1} -w^*)^\top H (w_{t-1} -w^*)\\
& \leq &\E\|w_{t-1} - w^*\|^2 -\gamma\E\|w_{t-1} - w^*\|_H^2 \\
& \leq &(1-\gamma\mu)\E\|w_{t-1} - w^*\|^2 \, ,
\end{eqnarray*}
which completes the proof.
\end{proof}

\noindent
{\bf Variance.\/}
Now suppose $w_0=w^*$. Define the covariance matrix:
\[
C_t:=\E[(w_t - w^*) (w_t - w^*)^\top| w_0 = w^*]
\]
Using the recursion, $w_t-w^* = B_t (w_{t-1}-w^*) + \gamma \xi_t $,
\begin{equation}\label{eq:cov_update}
C_{t+1} = C_t - \gamma H C_t - \gamma C_t H 
+ \gamma^2 \E[(x^\top C_t x) x x^\top] + \gamma^2 \Sigma
\end{equation}
which follows from:
\[
\E[(w_t - w^*) \xi_{t+1}^\top]=0 \, , \textrm{ and } \, \,
\E[(x_{t+1} x_{t+1}^\top) (w_t - w^*) \xi_{t+1}^\top] = 0
\]
(these hold since $w_t - w^*$ is mean $0$ and both $x_{t+1}$ and $\xi_{t+1}$
are independent of $w_t - w^*$).

\begin{lemma} \label{lem:cov_infty}
Suppose $w_0=w^*$. There exists a unique $C_\infty$ such that:
\[
0 = C_0 \preceq C_1 \preceq \cdots \preceq C_\infty
\]
where $C_\infty$ satisfies:
\begin{equation}\label{eq:cov_infty}
C_\infty = C_\infty - \gamma H C_\infty - \gamma C_\infty H 
+ \gamma^2 \E[(x^\top C_\infty x) x x^\top] + \gamma^2 \Sigma \, .
\end{equation}
\end{lemma}

\begin{proof}
By recursion,
\begin{eqnarray*}
w_t - w^* & = & B_t (w_{t-1} - w^*) + \gamma \xi_t\\
& = & \gamma 
\left(\xi_t + B_t\xi_{t-1} +\cdots+ B_t \cdots B_2 \xi_1\right)\, .
\end{eqnarray*}
Using that $\xi_t$ is mean zero and independent of $B_{t'}$ and $\xi_{t'}$ for $t<t'$,
\begin{eqnarray*}
C_t = \gamma^2 \left(\E[\xi_t \xi_t^\top] +\E[B_t \xi_{t-1} \xi_{t-1}^\top  B_t]+\cdots
+ \E[B_t \cdots B_2 \xi_1 \xi_1^\top  B_2^\top\cdots B_t^\top]
        \right)
\end{eqnarray*}
Now using that $\E[\xi_1 \xi_1^\top]=\Sigma$ and that $\xi_t$ and
$B_{t'}$ are independent (for $t\neq t'$),
\begin{eqnarray*}
C_t &=& \gamma^2 \left(\Sigma+\E[B_2 \Sigma  B_2]+\cdots
+ \E[B_t \cdots B_2 \Sigma  B_2^\top\cdots B_t^\top] \right)\\
&=& C_{t-1}+ \gamma^2\E[B_t \cdots B_2 \Sigma B_2^\top\cdots B_t^\top] \nonumber
\end{eqnarray*}
which proves $C_{t-1} \preceq C_t$. 

To prove the limit exists, it suffices to first argue the trace of $C_t$ is
uniformly bounded from above, for all $t$. By taking the trace of update rule,
Equation~\ref{eq:cov_update}, for $C_t$, 
\[ 
\Tr(C_{t+1}) = \Tr(C_t)-2\gamma \Tr(HC_t)+\gamma^2 \Tr(\E[(x^\top
C_t x) x x^\top]) + \gamma^2 \Tr(\Sigma) \, . 
\]
Observe:
\begin{equation}\label{eq:cov_bound}
\Tr(\E[(x^\top C_t x) x x^\top]) = \Tr(\E[(x^\top C_t x) \|x\|^2])
= \Tr(C_t \E[\|x\|^2 x x^\top]) \leq R^2\Tr(C_t H) 
\end{equation}
and, using $\gamma\leq 1/R^2$,
\[ 
\Tr(C_{t+1}) \leq \Tr(C_t)-\gamma \Tr(HC_t) + \gamma^2 \Tr(\Sigma)
\leq (1-\gamma \mu)\Tr(C_t) + \gamma^2 \Tr(\Sigma) \leq \frac{\gamma
  \Tr(\Sigma)}{\mu} \, . 
\]
proving the uniform boundedness of the trace of $C_t$.
Now, for any fixed $v$, the limit of $v^\top C_t v$ exists, by the 
monotone convergence theorem. From this, it follows that every entry
of the matrix $C_t$ converges.  
\end{proof}

\begin{lemma}
Define:
\begin{align*}
	\overline{w}_T &:= \frac{1}{T} \sum_{t=0}^{T-1} w_t \, .
\end{align*}
and so:
\begin{align*}
\frac{1}{2} \E [\| \overline{w}_T - w^* \|_H^2 | w_0=w^*]
&\leq \frac{\Tr(C_\infty)}{\gamma T} 
\end{align*}
\end{lemma}

\begin{proof}
Note
	\begin{align*}
		\E[ (\overline w_T - w^*) (\overline w_T - w^*)\T | w_0 = w^*] 
			=& \frac{1}{T^2} \sum_{t=0}^{T-1} \sum_{t'=0}^{T-1} \E[ (w_t - w^*) (w_{t'} - w^*)\T | w_0 = w^* ] \\
			\preceq & \frac{1}{T^2} \sum_{t=0}^{T-1} \sum_{t'=t}^{T-1}  
				\bigg( \E[ (w_t - w^*) (w_{t'} - w^*)\T  | w_0 = w^* ] + \\
				& \qquad \qquad \qquad \E[ (w_{t'} - w^*) (w_t - w^*)\T  | w_0 = w^* ] 
				\bigg) \, ,
	\end{align*}
	double counting the diagonal terms $\E[ (w_t - w^*) (w_t - w^*)\T  | w_0 = w^* ] \succeq 0$.
	For $t \le t'$, $\E[ (w_{t'} - w^*) | w_0 = w^* ] = (\Id - \gamma  H)^{t'-t} \E[ (w_t - w^*)   | w_0 = w^* ]$.
	To see why, consider the recursion $w_t - w^* = (\Id - \gamma x_t x_t\T) (w_{t-1} - w^*) - \gamma\xi_t$
	and take expectations to get $\E[w_t - w^* | w_0 = w^*]  = (\Id - \gamma H) \E[ w_{t-1} - w^*  | w_0 = w^* ]$ since the sample $x_t$
	 is independent of the $w_{t-1}$.
	From this, 
	\begin{align*}
		\E[ (\overline w_T - w^*) (\overline w_T - w^*)\T | w_0 = w^*] 
			&\preceq \frac{1}{T^2} \sum_{t=0}^{T-1} \sum_{\tau = 0}^{T-t-1} (I - \gamma H)^\tau C_t + C_t (\Id - \gamma H)^\tau \, ,
	\end{align*}
	and so,
	\begin{align*}
		\E [\| \overline{w}_T - w^* \|_H^2  |  w_0 = w^*] 
			&= \Tr \big( H \E[ (\overline w_T - w^*) (\overline w_T - w^*)\T | w_0 = w^*] \big) \\
			&\le \frac{1}{T^2} \sum_{t=0}^{T-1}\sum_{\tau = 0}^{T-t-1}  \Tr \big( H (\Id - \gamma H)^\tau C_t \big) 
				+ \Tr \big( C_t (\Id - \gamma H)^\tau H \big) \, .
	\end{align*}
	Notice that $H(\Id-\gamma H)^\tau  = (\Id - \gamma H)^\tau H$ for any non-negative integer $\tau$. 
	Since $H \succ 0$ and $I - \gamma H \succeq 0$, $H(\Id - \gamma H)^\tau \succeq 0$
	because the product of two commuting PSD matrices is PSD. Also note that for PSD matrices $A, B$, 
	$\Tr AB \ge 0$. Hence,
	\begin{align*}
		\E [\| \overline{w}_T - w^* \|_H^2  |  w_0 = w^*] 
			&\le \frac{2}{T^2} \sum_{t=0}^{T-1} \sum_{\tau = 0}^{\infty} \Tr\big( H  (\Id - \gamma H)^\tau C_t \big)  & \\
			&= \frac{2}{T^2} \sum_{t=0}^{T-1}  \Tr\big( H  (\sum_{\tau = 0}^{\infty} (\Id - \gamma H)^\tau) C_t \big)  & \\
			&= \frac{2}{T^2} \sum_{t=0}^{T-1}  \Tr\big( H  (\gamma H)^{-1} C_t \big)  & (*) \\
			&= \frac{2}{\gamma T^2} \sum_{t=0}^{T-1}  \Tr (C_t)  & \\
			&\le \frac{2}{\gamma T} \cdot \Tr (C_\infty) \,,
	\end{align*}
	from lemma~\ref{lem:cov_infty} where $(*)$ followed from 
	\begin{align*}
		(\gamma H)^{-1} = ( \Id -  (\Id - \gamma H))^{-1} = \sum_{\tau = 0}^\infty (\Id - \gamma H)^\tau \, ,
	\end{align*}
	and the series converges because $\Id - \gamma H \prec \Id$.
\end{proof}

\subsection{Stationary Distribution Analysis}

Define two linear operators on symmetric matrices, $\cS$ and
$\cT$ --- where $\cS$ and
$\cT$ can be viewed as matrices acting on ${d+1 \choose 2}$
dimensions --- as follows:
\[
\cS \circ M := \E[(x^\top M x) x x^\top] \, , \quad \quad \cT \circ M :=
H M + M H \, .
\]
With this, $C_\infty$ is the solution to:
\begin{equation}\label{C:def}
\cT \circ C_\infty = \gamma \cS \circ C_\infty +\gamma \Sigma
\end{equation}
(due to Equation~\ref{eq:cov_infty}).

\begin{lemma}\label{lemma:crude}
(Crude $C_\infty$ bound) $C_\infty$  is bounded as:
\[
C_\infty \preceq \frac{\gamma \|\Sigma\|_H}{1-\gamma R^2} \, \Id \, .
\]
\end{lemma}

\begin{proof}
Define one more linear operator as follows:
\[
\widetilde \cT \circ M := \cT \circ M - \gamma H M H = H M + M H - \gamma H M H \,.
\]
The inverse of this operator can be written as:
\[
\widetilde \cT^{-1} \circ M =\gamma \sum_{t=0}^\infty (\Id -
\gamma\widetilde \cT)^t \circ M
= \gamma \sum_{t=0}^\infty
(\Id-\gamma H)^t M (\Id-\gamma H)^t \, .
\]
which exists since the sum converges due to fact that $0\preceq\Id-\gamma H \prec
\Id$.  

A few inequalities are helpful: If $0\preceq M \preceq M'$,
then
\begin{equation}\label{eq:1}
0\preceq \widetilde \cT^{-1} \circ M \preceq \widetilde \cT^{-1} \circ
M' \, ,
\end{equation}
since
\[ 
\widetilde \cT^{-1} \circ M = 
\gamma \sum_{t=0}^\infty (\Id-\gamma
H)^t M (\Id-\gamma H)^t \preceq 
\gamma \sum_{t=0}^\infty (\Id-\gamma
H)^t M' (\Id-\gamma H)^t
= \widetilde \cT^{-1} \circ M' \, ,
\]
(which follows since $0\preceq \Id-\gamma H$).  Also, if $0\preceq M \preceq M'$,
then 
\begin{equation}\label{eq:2}
0 \preceq \cS \circ M \preceq \cS \circ M' \, ,
\end{equation}
which implies:
\begin{equation}\label{eq:3}
0 \preceq \widetilde \cT^{-1} \circ \cS \circ M \preceq \widetilde \cT^{-1}
\circ \cS \circ M'  \, .
\end{equation}
The following inequality is also of use:
\[
\Sigma \preceq \|H^{-1/2} \Sigma
H^{-1/2}\| H = \|\Sigma\|_H H \, .
\]

By definition of $\widetilde \cT$, 
\[
\widetilde \cT \circ C_\infty = \gamma \cS \circ C_\infty +\gamma
\Sigma - \gamma H C_\infty H \, .
\]
Using this and Equation~\ref{eq:1},
\begin{eqnarray*}
C_\infty
&=& \gamma \widetilde \cT^{-1} \circ \cS \circ C_\infty
+\gamma \widetilde \cT^{-1} \circ \Sigma - \gamma \widetilde \cT^{-1} \circ( H
    C_\infty H)\\
&\preceq& \gamma \widetilde \cT^{-1} \circ \cS \circ C_\infty
+\gamma \widetilde\cT^{-1} \circ \Sigma\\
&\preceq & \gamma \widetilde \cT^{-1} \circ \cS \circ C_\infty
+\gamma \|\Sigma\|_H \widetilde \cT^{-1} \circ H \, .
\end{eqnarray*}
Proceeding recursively by using Equation~\ref{eq:3}, 
\begin{eqnarray*}
C_\infty
&\preceq & (\gamma \widetilde \cT^{-1} \circ \cS)^2 \circ C_\infty 
+ \gamma \|\Sigma\|_H (\gamma \widetilde \cT^{-1} \circ \cS ) \circ \widetilde \cT^{-1} \circ H 
+\gamma \|\Sigma\|_H \widetilde \cT^{-1} \circ H\\
&\preceq&  \gamma \|\Sigma\|_H \sum_{t=0}^\infty (\gamma \widetilde \cT^{-1} \circ \cS)^t \circ
\widetilde \cT^{-1} \circ H \, .
\end{eqnarray*}

Using
\begin{align*}
\cS \circ \Id & \preceq  R^2 H 
\end{align*}
and
\begin{align*}
& \widetilde \cT^{-1} \circ H \\
& = 
\gamma \sum_{t=0}^\infty
(\Id-\gamma H)^{2t} H 
= \gamma \sum_{t=0}^\infty
(\Id-\gamma 2 H +\gamma^2 H)^t H 
\preceq \gamma \sum_{t=0}^\infty
(\Id-\gamma H)^t H 
= \gamma (\gamma H)^{-1} H
= \Id
\end{align*}
leads to
\[
C_\infty \preceq \gamma \|\Sigma\|_H \sum_{t=0}^\infty (\gamma R^2)^t
\Id = \frac{\gamma \|\Sigma\|_H}{1-\gamma R^2} \, \Id \, ,
\]
which completes the proof.
\end{proof}

\begin{lemma}
(Refined $C_\infty$ bound) The $\Tr(C_\infty) $  is bounded as:
\[
\Tr(C_\infty) \leq 
\frac{\gamma}{2}\Tr(H^{-1} \Sigma ) +
\frac{1}{2} \, \frac{\gamma^2R^2}{1-\gamma R^2} \, d \|\Sigma\|_H 
\]
\end{lemma}

\begin{proof}
From Lemma~\ref{lemma:crude} and Equation~\ref{eq:2},
\[
\cS \circ C_\infty \preceq \frac{\gamma \|\Sigma\|_H}{1-\gamma R^2} \,
\cS \circ \Id \preceq \frac{\gamma R^2 \|\Sigma\|_H}{1-\gamma R^2} \,
H \, .
\]
Also, from Equation~\ref{eq:cov_infty}, $C_\infty$ satisfies:
\[
H C_\infty + C_\infty H =
\gamma \cS \circ C_\infty + \gamma \Sigma \, .
\]
Multiplying this by  $H^{-1}$ and taking the trace
leads to:
\begin{eqnarray*}
\Tr(C_\infty) &=& \frac{\gamma}{2} \Tr (H^{-1} \cdot (\cS \circ
                  C_\infty) )
  +\frac{\gamma}{2}\Tr(H^{-1} \Sigma )\\
&\leq & \frac{1}{2} \, \frac{\gamma^2R^2}{1-\gamma R^2} \, \|\Sigma\|_H \,  \Tr (H^{-1} H)
  +\frac{\gamma}{2}\Tr(H^{-1} \Sigma )\\
&= & \frac{1}{2} \, \frac{\gamma^2R^2}{1-\gamma R^2} \,  d \|\Sigma\|_H
  +\frac{\gamma}{2}\Tr(H^{-1} \Sigma )
\end{eqnarray*}
which completes the proof.
\end{proof}

\subsection{Completing the proof of Theorem~\ref{theorem:main}}


\begin{proof}
The proof of the theorem is completed by applying the developed
lemmas.  For the bias term, using convexity leads to:
\begin{eqnarray*}
\frac{1}{2} \E [ \|\overline{w}_{t:T} - w^*\|_H^2  | \xi_0=\cdots \xi_T=0]
&\leq&
\frac{1}{2}  R^2 \E [ \|\overline{w}_{t:T} - w^*\|^2  | \xi_0=\cdots \xi_T=0]\\
&\leq&
\frac{1}{2} \frac{R^2}{T-t} \sum_{t'=t}^{T-1} \E [ \|w_{t'} - w^*\|^2  | \xi_0=\cdots \xi_T=0]\\
&\leq &
\frac{1}{2} \exp(-\gamma \mu t) R^2 \|w_0- w^*\|^2 \, .
\end{eqnarray*}
For the variance term, observe that
\begin{eqnarray*}
\frac{1}{2} \E [\| \overline{w}_{t:T} - w^* \|_H^2 |w_0=w^*]
\leq
\frac{\Tr(C_\infty)}{\gamma (T-t)} 
\leq
\frac{1}{T-t} 
\left(\frac{1}{2}\Tr(H^{-1} \Sigma ) +
\frac{1}{2} \, \frac{\gamma R^2}{1-\gamma R^2} \, d \|\Sigma\|_H \right) \, ,
\end{eqnarray*}
which completes the proof.
\end{proof}

\subparagraph*{Acknowledgements.}

Sham Kakade acknowledges funding from the Washington Research Foundation Fund for Innovation in Data-Intensive Discovery. The
authors also thank Zaid Harchaoui for helpful discussions.

\bibliography{acceleratedStochasticApproximation}

\end{document}